%% file: paper.tex
\documentclass[letterpaper]{article}

\usepackage{aaai18}

\usepackage{times}
\usepackage{helvet}
\usepackage{courier}

\usepackage{algorithm}
\usepackage{algpseudocode}
\usepackage{amsmath}
\usepackage{amssymb}
\usepackage{amsthm}
\usepackage{array}
\usepackage[english]{babel}
\usepackage{booktabs}
\usepackage{enumitem}
\usepackage{graphicx}
\usepackage{hyphenat}
\usepackage{mathtools}
\usepackage{microtype}
\usepackage{ragged2e}
\usepackage[subrefformat=parens]{subcaption}
\usepackage{xcolor}

\newcommand\citealt[1]{\citeauthor{#1} \citeyear{#1}}

\newcommand\ie{i.\,e.}
\newcommand\eg{e.\,g.}
\newcommand\st{s.\,t.}
\newcommand\wrt{w.\,r.\,t.}

\newcommand\Algorithm[1]{\texttt{\footnotesize #1}}
\newcommand\Dataset[1]{\texttt{\footnotesize #1}}

\theoremstyle{definition}
\newtheorem{definition}{Definition}

\theoremstyle{plain}
\newtheorem{theorem}{Theorem}
\newtheorem{lemma}{Lemma}
\newtheorem{proposition}{Proposition}

\DeclareMathOperator*\argmin{arg\,min}
\DeclareMathOperator*\argmax{arg\,max}

\DeclareMathOperator*\nb{Nb}

\newcommand\R{\mathbb R}
\newcommand\BR{\mathbb R}
\newcommand\SG{\mathcal G}
\newcommand\SA{\mathcal A}
\newcommand\SB{\mathcal B}
\newcommand\SAB{\mathcal{AB}}
\newcommand\SV{\mathcal V}
\newcommand\SE{\mathcal E}
\newcommand\SC{\mathcal C}
\newcommand\SX{\mathcal X}
\newcommand\SI{\mathcal I}
\newcommand\SAC{\mathcal S}
\newcommand\DA{{\partial A}}

\newcommand\TRWS{\mbox{TRW-S}}

\newcommand\Paragraph[1]{\addvspace{2pt}\noindent\textbf{#1}\hskip 4pt plus 4pt minus 2pt\relax}

\newcolumntype{L}[1]{>{\RaggedRight}p{#1}}
\newcolumntype{R}[1]{>{\RaggedLeft}p{#1}}

\title{Exact MAP-Inference by Confining Combinatorial Search with LP Relaxation}
\author{%
  Stefan Haller\footnotemark[2],
  Paul Swoboda\footnotemark[3],
  Bogdan Savchynskyy\footnotemark[2]\\
  \footnotemark[2] University of Heidelberg,\hspace{1em}%
  \footnotemark[3] IST Austria\\
  \texttt{\small stefan.haller@iwr.uni-heidelberg.de}%
}

\begin{document}

\maketitle

\begin{abstract}
We consider the MAP-inference problem for graphical models, which is a valued constraint satisfaction problem defined on real numbers with a natural summation operation.
We propose a family of relaxations~(different from the famous Sherali-Adams hierarchy), which naturally define lower bounds for its optimum.
This family always contains a tight relaxation and we give an algorithm able to find it and therefore, solve the initial non-relaxed NP-hard problem.

The relaxations we consider decompose the original problem into two non-overlapping parts: an easy LP-tight part and a difficult one.
For the latter part a combinatorial solver must be used.
As we show in our experiments, in a number of applications the second, difficult part constitutes only a small fraction of the whole problem.
This property allows to significantly reduce the computational time of the combinatorial solver and therefore solve problems which were out of reach before.
\end{abstract}

\section{Introduction}

This paper focuses on \emph{energy minimization} or \emph{maximum a posteriori~(MAP) inference} for undirected graphical models. This problem is closely related to weighted and valued constraint satisfaction.
In the most common pairwise case it amounts to minimizing a partially separable function $E_{\SG}$ taking real values on a discrete set of finite-valued vectors~$\SX_{\SV}$\footnote{We rigorously define notation in Section ``Preliminaries''.}
\begin{equation}
  \label{equ:energy-min}
  \min_{x\in\SX_{\SV}} \left[ E_{\SG}(\theta,x):= \sum_{\mathclap{u \in \SV}} \theta_u(x_u) + \sum_{\mathclap{uv \in \SE}} \theta_{uv}(x_u, x_v) \right].
\end{equation}

The problem is known to be NP-hard~(\eg\ \citealt{li2016complexity}), and therefore a number of approximate algorithms were proposed to this end.
In contrast, our goal is an efficient method able to solve large-scale, but mostly simple problem instances \emph{exactly}.
Such instances typically arise in computer vision, machine learning and other areas of artificial intelligence.
Although approximate methods often provide reasonable solutions, having an exact solver can be quite critical at the modeling stage, when one has to differentiate between modeling and optimization errors.
In this case one usually resorts to either specialized combinatorial solvers~(see references in~\citealt{kappes2015comparative}; \citealt{hurley2016multi}) or off-the-shelf integer linear program~(ILP) solvers like CPLEX~\cite{cplex} or Gurobi~\cite{gurobi}.
However, neither specialized nor off-the-shelf solvers scale well, as the problem instances get larger.
Our method is able to use the fact that a linear program~(LP) relaxation of the problem is ``almost'' tight, \ie\ the obtained solution is close to the optimal one.
It restricts application of an exact solver to a small fraction of the problem, where the LP relaxation is not tight and yet obtains a provably optimal solution to the whole problem.
This allows to solve problems for which no efficient solving technique was available.

\Paragraph{Related work}
\emph{LP relaxations} are an important building block for a number of algorithms addressing the MAP-inference problem~\eqref{equ:energy-min}.
It was probably first considered in~(\citealt{shlezinger1976syntactic}; see~\citealt{werner2007linear} for the recent review) both in its primal and dual form.
The notion of \emph{reparametrization}~(known also as \emph{equivalent transformations} or \emph{equivalence preserving transformations}) was introduced in the same work as well.
Although the bound provided by the LP relaxation is often good, the class of problems, where it is tight, is limited~(see~\citealt{kolmogorov2015power}).
Practically important problems from this class are mainly those having acyclic structure or submodular costs.
Therefore a number of works were devoted to \emph{cutting plane} techniques to tighten the relaxation~(\eg\ \citealt{koster1998partial}; \citealt{sontag2007cutting}; \citealt{degivry2017clique}).
Sometimes the tightening itself may lead to an exact solution, however, in general it is accomplished with \emph{branch-and-bound} or $A^*$ algorithms.
The most prominent representative of the first class are the DAOOPT~\cite{marinescu2005and,otten2010toward} and Toulbar2~\cite{cooper2010soft} solvers.
The latter has recently shown impressive results on a number of benchmarks~\cite{hurley2016multi}.
In contrast, the $A^*$ algorithm so far was mainly used in specific applications~(\eg\ \citealt{bergtholdt2010study}).

Recently developed LP-relaxation-based \emph{partial optimality methods}~(\eg\ \citealt{shekhovtsov2014maximum}; \citealt{shekhovtsov2015maximum}; \citealt{Swoboda2016}) can find optimal labels for a significant part of variables without solving the combinatorial problem~\eqref{equ:energy-min}.
Afterwards, a combinatorial solver can be applied to the rest of the variables to obtain a complete solution.
These methods work well if the pairwise costs $\theta_{uv}$ play the role of a ``smoothing regularizer'' by slightly penalizing differences in values in neighboring variables $u$ and $v$.
However, they struggle as the pairwise costs get more weight and move towards ``hard constraints'', when some pairs of variable values are strongly penalized or even forbidden.

The \emph{CombiLP} method~\cite{savchynskyy2013global} is the closest to our work.
It iteratively splits the problem into a ``simple'' and a ``difficult'' part based on consistency of reparametrized unary and pairwise costs, known as~(virtual) arc-consistency~(see~\citealt{werner2007linear}), and checks for agreement of their solutions.
The ``simple'' part is addressed with a dedicated LP solver, whereas the ``difficult'' one is solved with an ILP method.
Although CombiLP has shown promising results on the OpenGM benchmark~\cite{kappes2015comparative}, its usage is beneficial for sparse graphical models only, when~$|\SE|\ll |\SV|^2$.

\Paragraph{Contribution}
Based on CombiLP, we propose a method, which is not restricted to sparse models.
Similar to CombiLP, we split the problem into LP and ILP parts based on local consistency properties.
Our new consistency criterion guarantees that the concatenation of the obtained LP and ILP solutions is optimal for the whole problem, given that the criterion is satisfied.
When the criterion is not satisfied, we increase the ILP subproblem and correspondingly decrease the LP one, like it is done in CombiLP.
There are several crucial differences to the CombiLP approach, however:
\begin{itemize}[nosep]
\item Our ``difficult'' ILP subproblem is kept much more compact, which is critical for densely-connected graphs.
This leads to substantial computational savings.
\item Our optimality criterion is stronger than those of CombiLP: Satisfaction of CombiLP's criterion for a given splitting implies satisfaction of ours.
\end{itemize}

Additionally, we treat the problem of an initial reparametrization suitable for the used splitting criterion and propose a method, which allows to use \emph{arbitrary} dual LP solvers within our algorithm, whereas the CombiLP implementation has a \emph{fixed} dedicated LP solver.
This allowed us to choose a more efficient LP solver and to significantly~(up to 18 times) speed up the original CombiLP implementation.

Finally, our criterion and implementation\footnote{Code is available at \texttt{github.com/fgrsnau/combilp}.} are also able to deal with higher order models, which intrinsically have a higher connectivity.
We show efficacy of our method on publicly available benchmarks from computer vision, machine learning and bio-imaging.

\section{Preliminaries}\label{sec:preliminaries}

\Paragraph{Graphical Models and MAP-inference}
Let $\mathcal G = (\SV, \SE)$ be an undirected graph with \emph{the set of nodes} $\SV$ and \emph{the set of edges} $\SE$.
The neighborhood of $u\in\SV$ is defined as $\nb(u): = \{ v \in \SV \mid uv \in \SE \}$.
Each node $v \in \SV$ is associated with a finite \emph{set of labels} $\SX_v$.
For any subset $\SV'\subseteq\SV$ of graph nodes the Cartesian product $\mathcal X_{\SV'} = \prod_{v \in \SV'} \mathcal X_v$ defines \emph{the set of labelings} of the subset $\SV' \subseteq \SV$, when each node from~$\SV'$ is assigned a label.
This includes also the special cases $\SV'=\{u,v\}\in\SE$ and $\SV'=\SV$ denoted as $\SX_{uv}$ and $\SX_{\SV}$ respectively.
We assume that $\argmin$ stands for the set of minimal elements. At the same time, when used with~``$=$''~or~``$:=$'' operators, it returns some element from this set.

Let ${\SI=\{ s\in\SX_u \mid u\in\SV\}\cup\{ (s,t)\in\SX_{uv} \mid uv\in\SE\}}$ be the set of indices enumerating all labels and label pairs in neighboring graph nodes.
For each node and edge the \emph{cost} functions $\theta_u : \SX_u \to \R$, $u\in\SV$ and $\theta_{uv} : \SX_{uv} \to \BR$, $uv\in\SE$ assign a cost to a label or label pair respectively.
The vector $\theta\in\BR^{\SI}$ contains all values of the functions $\theta_u$ and $\theta_{uv}$ as its coordinates.

\Paragraph{ILP formulation and LP relaxation}
One way to address the MAP-inference problem~\eqref{equ:energy-min} is to consider its ILP formulation~(see \eg~\citealt{shlezinger1976syntactic}; \citealt{werner2007linear})
\begin{align}
 &\min\nolimits_{\mu\ge 0}\langle\theta,\mu\rangle\nonumber\\
 &\sum\nolimits_{s\in\SX_u}\mu_u(s)=1,\ u\in\SV \label{equ:LP-formulation} \\
 &\sum\nolimits_{s\in\SX_u}\mu_{uv}(s,t)=\mu_v(t),\ uv\in\SE, t\in\SX_v \nonumber\\
 & \mu\in\{0,1\}^{\SI}\label{equ:integrality-constraints}
\end{align}
A natural LP relaxation is obtained by omitting the integrality constraints~\eqref{equ:integrality-constraints}.
The resulting LP~\eqref{equ:LP-formulation} is known as a \emph{local polytope}~\cite{werner2007linear} or simply \emph{an LP relaxation} of~\eqref{equ:energy-min}.
We will call the problem~\eqref{equ:energy-min} LP-tight, if the optimal values of~\eqref{equ:energy-min} and its LP relaxation~\eqref{equ:LP-formulation} coincide.
This also implies that there is an integer solution to the relaxed problem~\eqref{equ:LP-formulation}.
We will say that the LP relaxation {\em has an integer solution} in a node $u$ if there is $s\in\SX_u$ such that $\mu_u(s)=1$.
Due to constraints of~\eqref{equ:LP-formulation} it implies that $\mu_u(s')=0$ for $s'\in\SX_u\setminus \{s\}$.

Linear programs of the form~\eqref{equ:LP-formulation} are as difficult as linear programs in general~\cite{prusa2013universality} and therefore obtaining exact solutions for large-scale instances may require significant time.
However, there are fast specialized solvers~(\eg~\citealt{kolmorogrov2006convergent}; \citealt{cooper2008virtual}) returning approximate dual solutions of~\eqref{equ:LP-formulation}.

\Paragraph{Partial Optimality Observation}
Practical importance of the LP relaxation~\eqref{equ:LP-formulation} is based on the fact that often most coordinates of its~(approximate) relaxed solution are assigned integer values.
The non-integer coordinates can be rounded~\cite{ravikumar2010message} and the resulting labeling can be used as if it was a solution of the non-relaxed problem.
A number of problems have been successfully addressed with this type of methods~\cite{kappes2015comparative}.
However, apart from special cases~(\eg\ \citealt{boros2002pseudo}; \citealt{rother2007optimizing}) there is no guarantee that the integer coordinates keep their values in an optimal solution of the non-relaxed problem.

Even though there is no guarantee that the rounded integer solution is a sensible approximation for the optimal solution, empirical tests have shown that usually many integer coordinates coincide with the ones found in the optimal solution.
This is a purely practical observation with little theoretical background.
Nevertheless, this observation can be used to address the non-relaxed problem efficiently and it is a basis of our method.
An alternative, the \emph{partial optimality} approach was pursued by \eg\ \citealt{shekhovtsov2015maximum}; \citealt{Swoboda2016}.
We will provide a corresponding empirical comparison later in the paper.

\section{Idea of the Algorithm}

\Paragraph{Graph partition}
A subgraph $\SG'=(\SV',\SE')$ of the graph $\SG=(\SV,\SE)$ is called \emph{induced} by the set of its nodes $\SV'$, if $\SE'=\{uv\in\SE \mid u,v \in\SV'\}$, \ie\ the set $\SE'$ of its edges contains all edges from $\SE$ connecting nodes from $\SV'$.

The subgraphs $\SA=(\SV_{\SA},\SE_{\SA})$ and $\SB=(\SV_{\SB},\SE_{\SB})$ are called \emph{partition} of the graph $\SG$, if $\SV_{\SA}\cup\SV_{\SB}=\SV$, $ \SV_{\SA}\cap\SV_{\SB}=\emptyset$ and $\SA$ and $\SB$ are induced by $\SV_{\SA}$ and $\SV_{\SB}$ respectively.
The subgraph $\SB$ as complement to $\SA$ will be denoted as $\SG\setminus \SA$.
The other way around, $\SG\setminus \SB$ stands for $\SA$, if $\SA,\SB$ is a partition of $\SG$.
Notation~$\SE_{\SAB}$ will be used for the set of edges connecting~$\SV_{\SA}$ and~$\SV_{\SB}$: $\SE_{\SAB}=\{uv\in\SE \mid u\in \SV_{\SA}, v\in\SV_{\SB}\}$.

In the following, we will show how to partition the problem graph $\SG$ into (i)~an easy part with subgraph $\SA$, which can be solved exactly with approximate LP solvers and (ii)~a difficult part with subgraph $\SB$, which will require an ILP solver.

\Paragraph{Lower bound induced by partition}
Till the end of this section we will assume $\SA$, $\SB$ are a partition of a graph $\SG$.
For the sake of notation, when considering different subgraphs of~$\SG$ we will nevertheless use a cost vector $\theta$ corresponding to the master graph $\SG$, \ie\ $E_{\SA}(\theta,x)$ will stand for
$\sum_{u\in\SV_{\SA}}\theta_u(x_u)+ \sum_{uv\in\SE_{\SA}}\theta_{uv}(x_u,x_v)$, 
where $\theta\in\BR^{\SI}$.

Additionally, for $x'\in\SX_\SA$ and $x''\in\SX_\SB$, their concatenation $x' \circ x'' \in\SX_\SG$ will be defined as
\begin{equation}
  \label{equ:concatenation}
  (x' \circ x'')_v =
    \begin{cases}
      x'_v\,,  & v \in \SV_\SA\;,\\
      x''_v\,, & v \in \SV_\SB\;.
    \end{cases}
\end{equation}

Note that the energy function $E_\SG$ can be decomposed into subproblems on $\SA$ and $\SB$ and it holds
\begin{equation}\label{equ:energy-decomposition}
  E_\SG(\theta,x' \circ x'') = E_\SA(\theta,x') + E_\SB(\theta,x'') + \sum_{\mathclap{uv\in\SE_\SAB}}\theta_{uv}(x'_u, x''_v)
\end{equation}
and therefore,
\begin{multline}\label{equ:lower-bound}
  E_\SG(\theta,x' \circ x'') \ge E_{\SA}(\theta,x') + E_{\SB}(\theta,x'') +{} \\ {}+ \sum_{uv\in\SE_\SAB}\min_{(s,t)\in\SX_{uv}} \theta_{uv}(s, t)
\end{multline}
constitutes a lower bound for the energy function $E_{\SG}$.

\begin{proposition}[Sufficient optimality condition]\label{prop:lower-bound}
  The lower bound specified in~\eqref{equ:lower-bound} is tight if for all $uv\in\SE_\SAB$ it holds that $(x'_u,x''_v)\in\argmin_{(s, t)\in\SX_{uv}} \theta_{uv}(s, t)$, where ${x'\in\argmin_{x\in\SX_\SA}E_{\SA}(\theta,x)}$, ${x''\in\argmin_{x\in\SX_\SB}E_{\SB}(\theta,x)}$.
\end{proposition}

It is trivial to show that the labeling $x'\circ x''$ is optimal for $E_\SG$ if the lower bound~\eqref{equ:lower-bound} is tight.

When considering the set of all possible partitions of $\SG$ into $\SA$ and $\SB$ there is always at least one that leads to a tight lower bound~\eqref{equ:lower-bound}.
It corresponds to a trivial partition, where either the subgraph $\SA$ or $\SB$ is empty.
The first case corresponds to solving the whole problem with an ILP method, whereas the second one corresponds to the case when the LP relaxation is tight, \ie\ all coordinates of an LP solution are integer.

However, as our experimental evaluation shows, there exist often tight non-trivial partitions, with a large subgraph $\SA$ and a small subgraph $\SB$.

\Paragraph{Conceptual Algorithm} These partitions can be obtained for example by \emph{a conceptual} Algorithm~\ref{alg:dense-CombiLP-sketch}, which assigns all nodes of the graph having an integer solution to $\SA$ and all others to $\SB$.
After solving both subproblems one checks fulfillment of the sufficient optimality condition defined by Proposition~\ref{prop:lower-bound}.
Should the condition hold, the problem is solved.
Otherwise one increases the subproblem $\SB$~(and respectively decreases $\SA$) by including those nodes $u\in\SV_\SA$, where the condition $(x'_u, x''_v)\in\argmin_{(s, t)\in\SX_{uv}} \theta_{uv}(s, t)$ does not hold for at least one $v\in\SV_{\SB}$, in terms of Proposition~\ref{prop:lower-bound}.

\begin{algorithm}[tb]
  \newcommand\MultiState[1]{\State\parbox[t]{\linewidth-\algorithmicindent}{\hangindent=12pt #1}}
  \begin{algorithmic}[1]
   \State Solve LP relaxation~\eqref{equ:LP-formulation}
   \State Assign all nodes with integer solution to $\SA$
   \Repeat
   \State Set $\SB:=(\SG\setminus\SA)$
    \MultiState{Compute optimal labelings $x'$ and $x''$ on $\SA$ and $\SB$ respectively. Use LP solver for $\SA$ and ILP for $\SB$.}
    \If {(Prop.~\ref{prop:lower-bound} holds for $x'$ and $x''$)}
     \State \Return $(x' \circ x'')$
    \Else
      \State Move those $u$ from $\SA$ to $\SB$, where Prop.~\ref{prop:lower-bound} fails
    \EndIf
   \Until{$\SB=\SG$}
  \end{algorithmic}
  \caption{Conceptual Dense-CombiLP Algorithm}
  \label{alg:dense-CombiLP-sketch}
\end{algorithm}

\Paragraph{Relation to CombiLP}
Algorithm~\ref{alg:dense-CombiLP-sketch} differs from CombiLP~\cite{savchynskyy2013global} in one very important aspect.
Namely, the subgraphs used in CombiLP are overlapping, whereas ours are not.
This substantially improves performance of the method in cases when the graph $\SG$ has a high connectivity.
In later sections of this paper we will give a detailed theoretical and empirical comparison of the methods.

In the following, we will turn the conceptual algorithm into a working one.
In order to do so, we will give positive answers to a number of important questions:
\begin{itemize}[nosep]
\item Why and when is the subproblem on $\SA$ LP-tight? This is critical, since we assume $\SA$ to be close to $\SG$ in its size and therefore it must be solvable by a~(polynomial) LP method.
\item Can we avoid running an LP solver for $\SA$ in each iteration?
\item Can we use~(fast specialized) approximate LP solvers on $\SA$ instead of~(slow off-the-shelf) exact ones?
\item How to encourage conditions of Proposition~\ref{prop:lower-bound} to be fulfilled for a possibly small $\SB$?
\end{itemize}

Although our construction mostly follows the one given in~\cite{savchynskyy2013global}, we repeat it here to keep the paper self-contained.

\section{Theoretical Background}

\Paragraph{Reparametrization}
Decompositions of the energy function $E_\SG(\theta,x)$ into unary and pairwise costs are not unique, which is, there exist other costs $\theta'\in\BR^{\SI}$ such that  $E_\SG(\theta,x)=E_\SG(\theta',x)$ for all labelings $x\in\SX_{\SV}$.
It is known~(see \eg~\citealt{werner2007linear}) and straightforward to check that such \emph{equivalent} costs can be obtained with an arbitrary vector~$\phi=(\phi_{u,v}(s)\in\BR \mid u\in\SV,\ v\in\nb(u),\ s\in\SX_u)$ as follows:
\begin{align}\label{equ:reparametrization}
 \theta'_u(s)\equiv \theta^{\phi}_u(s) &:=\theta_u(s)-\sum_{v\in\nb(u)}\phi_{u,v}(s)\\
 \theta'_{uv}(s,t) \equiv \theta^{\phi}_{uv}(s,t) & :=\theta_{uv}(s,t) + \phi_{u,v}(s)+\phi_{v,u}(t)\,.\nonumber
\end{align}
The costs~$\theta^{\phi}$ are called \emph{reparametrized} and the vector~$\phi$ is known as a \emph{reparametrization}.
Costs related by~\eqref{equ:reparametrization} are also called \emph{equivalent}.
In this sense, all vectors $\theta\in\BR^{\SI}$ can be split into equivalence classes according to~\eqref{equ:reparametrization}.
Other established terms for reparametrizations are \emph{equivalence preserving transformations}~\cite{cooper2004arc} and \emph{equivalent transformations}~\cite{shlezinger1976syntactic}.

\Paragraph{Dual Problem}
By swapping the $\min$ and $\sum$ operations in~\eqref{equ:energy-min} one obtains a lower bound $D(\theta) \le E_\SG(\theta,x)$ to the energy\footnote{It can be shown that this bound is in general less tight than~\eqref{equ:lower-bound}.}, which reads as
\begin{equation}\label{equ:LP-lower-bound}
  D(\theta):=\sum_{u\in\SV}\min_{s\in\SX_u}\theta_u(s) + \sum_{uv\in\SE}\min_{(s,t)\in\SX_{uv}}\theta_{uv}(s,t)\,.
\end{equation}
Although the energy $E_\SG(\theta,x)$ remains the same for all cost vectors from a given equivalence class~($E_\SG(\theta,x)=E_\SG(\theta^{\phi},x)$), the lower bound $D(\theta)$ is dependent on the reparametrization~($D(\theta)\neq D(\theta^{\phi})$).
Therefore, a natural maximization problem arises as maximization of the lower bound over all equivalent costs: $\max_{\phi}D(\theta^\phi)$.
It is known~(\eg\ \citealt{werner2007linear}) that this maximization problem is equivalent to the Lagrangian dual to the LP relaxation~\eqref{equ:LP-formulation}.
In turn, this implies that the minimum of~\eqref{equ:LP-formulation} coincides with the maximum of $D(\theta^\phi)$.
Therefore, one speaks about \emph{optimal reparametrizations} as those $\phi$, where the maximum is attained.
Apart from its lower bound property the function $D(\theta^\phi)$ is important because
(i)~function $D(\theta^{\phi})$ is concave \wrt\ $\phi$ as a sum of minima of linear functions;
(ii)~there exist many of scalable and efficient algorithms for its~(approximate) maximization, \eg~\cite{kolmorogrov2006convergent,cooper2008virtual}.

\Paragraph{Strict Arc-Consistency}
From a practical point of view it is important how an optimal reparametrization $\phi$ can be translated into a labeling, \ie\ into an~(approximate) solution of the energy minimization problem~\eqref{equ:energy-min}. The following definition plays a crucial role for this question in general and for our method in particular:
\begin{definition}[Strict arc-consistency]\label{def:sac}
  The node $u\in\SV$ is called \emph{strictly arc-consistent} \wrt\ the costs $\theta$ if there exists a label $x_u\in\SX_u$ and labels $x_v\in\SX_v$ for all $v\in\nb(u)$, such that it holds
  (i)~$\theta_u(x_u) < \theta_u(s)$ for all $s\in\SX_u\setminus\{x_u\}$; and
  (ii)~$\theta_{uv}(x_u,x_v) < \theta_{uv}(s,t)$ for all $(s,t)\in\SX_{uv}\setminus\{(x_u,x_v)\}$.

  The \emph{set of strictly arc-consistent nodes} is denoted by
  $\SAC(\theta): = \{ v\in\SV \mid v \text{ is strictly arc-consistent} \}$\,.%
\end{definition}
If all nodes are strictly arc-consistent \wrt\ the reparametrized costs $\theta^{\phi}$, then it is straightforward to check that $D(\theta^{\phi})=E_\SG(\theta^{\phi},x^*)=E_\SG(\theta,x^*)$, where
\begin{equation}\label{equ:locally-optimal-solution}
 x^*_u=\argmin_{s\in\SX_u}\theta^{\phi}_u(s)\,.
\end{equation}
In turn, this implies that $\phi$ is an optimal reparametrization and $x^*$ is an exact solution of the energy minimization problem~\eqref{equ:energy-min}.

\Paragraph{Reconstructing labeling from reparametrization}
Although there is no guarantee that the strict arc-consistency property holds for all nodes even with an optimal reparametrization, the rule~\eqref{equ:locally-optimal-solution} is still used to obtain an approximate minimizer for~\eqref{equ:energy-min} with arbitrary, also non-optimal reparametrizations $\phi$~(although, a number of more sophisticated rules were proposed, they are based on~\eqref{equ:locally-optimal-solution} and reduce to it if the strict arc-consistency holds for all nodes, see \eg\ \citealt{ravikumar2010message}).
Moreover, for an optimal reparametrization $\phi$, when the strict arc-consistency holds for a node~$u$, the complementary slackness conditions imply~(\eg\ \citealt{werner2007linear}) that strict arc-consistency of a node $u$ guarantees an integer solution of the LP relaxation in~$u$.

From the application point of view, an~(approximate) solution~\eqref{equ:locally-optimal-solution} is typically considered as good, if most of the nodes $u\in\SV$ satisfy the strict arc-consistency property.
At the same time, unless the strict arc-consistency holds for \emph{all} nodes, there is in general no theoretical guarantee that $x^*_u$ obtained as~\eqref{equ:locally-optimal-solution} coincide with the corresponding coordinate $y_u$ of an optimal solution $y=\argmin_{x\in\SX_\SV}E_\SG(\theta,x)$, even if the node~$u$ is strictly arc-consistent.

Algorithm~\ref{alg:dense-CombiLP} described in the next section, provides such guarantees by solving the non-relaxed minimization  problem~\eqref{equ:energy-min}.

\section{Detailed Algorithm Description}
Let us consider Algorithm~\ref{alg:dense-CombiLP}.
It differs from Algorithm~\ref{alg:dense-CombiLP-sketch} provided above in several aspects: Instead of solving the relaxed problem~\eqref{equ:LP-formulation} in the primal domain, it solves its dual formulation and resorts to the optimally reparametrized costs.
Strict arc-consistency is used in place of integrality to form the initial set $\SA$, which is justified by the fact that strict arc-consistency is sufficient for integrality.

The reparametrization step in line~\ref{step:alg:repa-problem} plays a crucial role for the whole method.
Due to this step,
solving the energy minimization problem on $\SA$ becomes trivial because of its strict arc-consistency.
It can be performed by selecting the best label in each node independently, according to~\eqref{equ:locally-optimal-solution}.
Therefore, \emph{there is no computational overhead} of resolving the problem on $\SA$ in each iteration.
Also, as more and more nodes from the initial subgraph $\SA$ move over to the subgraph $\SB$ their strict arc-consistency \emph{encourages} solution on $\SB$ to coincide with the locally optimal labels.
Moreover, instead of an optimal dual solution $\phi$ \emph{any}, also \emph{approximate}, \emph{non-optimal} reparametrization can be used.
According to Proposition~\ref{prop:lower-bound}, this does not affect correctness of Algorithm~\ref{alg:dense-CombiLP}.
Therefore, \emph{approximate solvers can be used} in line~\ref{step:alg:solve-dual} of the algorithm.
However, the better the dual solution is, the larger the set of strictly arc-consistent nodes $\SAC(\theta)$ is and therefore, the lower computational complexity of the ILP phase of the algorithm.
Finally, reparametrization of the costs typically speeds up the ILP solver in line~\ref{step:alg:ILP-B}, as it serves as preprocessing.

\begin{algorithm}[tb]
  \newcommand\MultiState[1]{\State\parbox[t]{\linewidth-\algorithmicindent}{\hangindent=12pt #1}}
  \begin{algorithmic}[1]
    \State Maximize the dual~\eqref{equ:LP-lower-bound} $\phi:=\argmax_{\psi}D(\theta^{\psi})$ ~\label{step:alg:solve-dual}
    \State Switch to the reparametrized costs: $\theta:=\theta^{\phi}$ \label{step:alg:repa-problem}
    \MultiState{Induce $\SA$ from $\SV_{\SA}:=\SAC(\theta)$ and get optimum on~$\SA$: ${x'_u:=\argmin\limits_{s\in\SX_u}\theta_u(s), u\in\SV_{\SA}}$}
    \Repeat
      \State Set $\SB:=(\SG\setminus\SA)$
      \State Compute an optimal labeling $x''$ on $\SB$ with ILP solver\label{step:alg:ILP-B}
      \State $O_{uv}:= \argmin\limits_{(s,t)\in\SX_{uv}} \theta_{uv}(s,t)$, see Prop.~\ref{prop:lower-bound}
      \If {$\forall uv\in\SE_{\SA\SB}: (x_u',x''_v)\in O_{uv}$}
        \State \Return $(x' \circ x'')$
      \Else
        \State $\SV_{\SA}:=\SV_{\SA}\setminus\{u \mid \exists uv\in\SE_\SAB: (x_u',x''_v)\notin O_{uv}\}$\label{step:alg:grow-B}
        \State Induce $\SA$ from $\SV_{\SA}$
      \EndIf
    \Until{$\SB=\SG$}
  \end{algorithmic}
  \caption{Dense-CombiLP Algorithm}
  \label{alg:dense-CombiLP}
\end{algorithm}

\section{Analysis of the Method}

\Paragraph{Family of Tight Partitions}
The proposition below that if the sufficient optimality criterion (Proposition~\ref{prop:lower-bound}) of Algorithm~\ref{alg:dense-CombiLP} is fulfilled for a partition $\SA,\SB$, then for any other partition $\SA',\SB'$ such that $\SB'\supseteq\SB$ the criterion holds as well:

\begin{proposition}\label{prop:greedy}
  Let $\SA$, $\SB$ and $\SA'$, $\SB'$ be partitions of $\SG$ \st~$\SB \subseteq \SB'$. Let also $\SV_\SA \subseteq \SAC(\theta)$.
  Let $x'$, $x''$ be the solutions of the MAP-inference problem~\eqref{equ:energy-min} on $\SA$ respectively $\SB$ and $y'$, $y''$ analogously for $\SA'$ and $\SB'$.
  If $x' \circ x''$ is the unique optimal assignment and it fulfills requirements of Proposition~\ref{prop:lower-bound}, then they are fulfilled for $y' \circ y''$ as well.
\end{proposition}
This property shows that there are potentially many partitions, which results in a tight bound and allows to apply a greedy strategy for growing the subgraph $\SB$ by adding all inconsistent nodes~(violating Proposition~\ref{prop:lower-bound}) at once, as it is done in line~\ref{step:alg:grow-B} of Algorithm~\ref{alg:dense-CombiLP}.

\Paragraph{Comparison to CombiLP}
As mentioned above, the CombiLP-method is very similar to ours, but uses a different optimality criterion.
Below we show that our criterion is in a certain sense stronger than theirs.
To this end, following~\cite{savchynskyy2013global}, we introduce the notion of a \emph{boundary complement subgraph}:

\begin{definition}[\citealt{savchynskyy2013global}]\label{def:boundary-complement}
 Let $\SA$ be an induced subgraph of $\SG$.
 A subgraph $\SB$ is called \emph{boundary complement to $\SA$ \wrt\ $\SG$} if it is induced by the set $(\SV_{\SG}\setminus \SV_{\SA})\cup \SV_{\DA}$, where
 $\SV_\DA=\{v\in\SV_{\SA} \mid \exists uv\in\SE_{\SG}: u\in\SV_{\SG}\setminus\SV_{\SA}\}$ is the set of nodes in $\SA$ incident to nodes outside $\SA$.
\end{definition}
The optimality criterion used in CombiLP reads:
\begin{theorem}[\citealt{savchynskyy2013global}]\label{thm:CombiLP}
 Let $\SA$ be a subgraph of $\SG$ and $\SB$ be its boundary complement \wrt\ $\SA$. Let $x_{\SA}$ and $x'_{\SB}$ be labelings minimizing $E_{\SA}$ and $E_{\SB}$ respectively \emph{and} let all nodes $v\in\SV_{\SA}$ be strictly arc-consistent. Then from
 \begin{equation}\label{consistencyCondition_equ}
 x_v=x'_v\ \mbox{ for all }\ v\in\SV_{\DA}
\end{equation}
it follows that the labeling $x^*$ with coordinates\\
${x^*_v=\left\{ 
\begin{array}{ll}
 x_v, & v\in\SA\\
 x'_v, & v\in\SB\setminus\SA\\
\end{array}
\right.,\ v\in\SV_{\SG}}$, is optimal on $\SG$.%minimizes energy $E_{\SG,\theta}$ on the graph~$\SG$.
\end{theorem}

As can be seen from comparing Proposition~\ref{prop:lower-bound} and Theorem~\ref{thm:CombiLP}, the main difference between the methods is that we use a partition of the graph $\SG$, \ie\ non-intersecting subgraphs, whereas the subgraphs in CombiLP are boundary complement and therefore intersect.

The following proposition states that the bounds produced by our method are at least as tight as those of CombiLP:

\begin{proposition}\label{prop:old-criterion}
  Let $\SA,\SB$ be a partition of a graph $\SG$ and $\SA',\SB$ be boundary complement for $\SG$ and $\SA\subseteq\SA'\subseteq\SAC(\theta)$.
  Let also $x'$, $x''$ be optimal labelings on $\SA'$ and $\SB$.
  If the condition~\eqref{consistencyCondition_equ} holds for $x'$ and $x''$, \ie\ $x'_v=x''_v$ for all $v\in\SV_{\DA'}$, then Proposition~\ref{prop:lower-bound} holds for $x'_\SA$ and $x''$ as well, where $x'_\SA$ is the restriction of $x'$ to the set $\SV_\SA$.
  In other words, for the same subgraph $\SB$ fulfillment of Theorem~\ref{thm:CombiLP} implies fulfillment of Proposition~\ref{prop:lower-bound}.
\end{proposition}

\section{Technical Details}

\begin{figure}
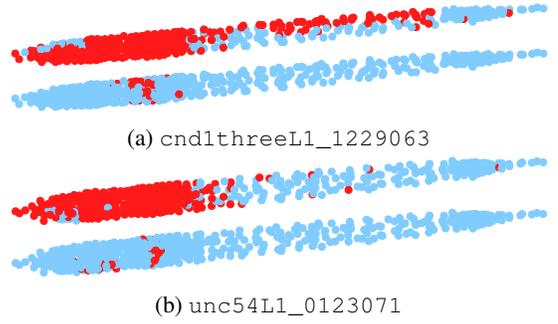

  \newcommand\TmpWormGraphic[2]{}%
  \newcommand\TmpWorm[1]{%
    \begin{subfigure}{\columnwidth}%
      \centering
      \includegraphics[width=.85\textwidth]{{imgs/worm_#1_final}.pdf}%
      \caption{\Dataset{\detokenize{#1}}}%
    \end{subfigure}}%

  \centering
  \TmpWorm{cnd1threeL1_1229063}\\
  \TmpWorm{unc54L1_0123071}

  \caption{
    Visualization of the maximal ILP subproblem for \Dataset{worms}, where dots correspond to cell nuclei.
    Red dots are part of the ILP subproblem~($\textcolor{red}{\bullet}\in\SV_\SB$) and for blue dots the solution of the LP-relaxation is used~($\textcolor[rgb]{.5,.8,1}{\bullet}\in\SV_\SA$).
    For both instances, the upper image shows the result for \Algorithm{clp} and the lower one corresponds to \Algorithm{dclp}, see the experimental evaluation section for information about the solvers.
  }
  \label{fig:masks}
\end{figure}

\Paragraph{Post-Processing of Reparametrization}
The maximum of the dual objective $D(\theta^{\phi})$ is typically non-unique.
Since $D$ is a concave function, the set of its maxima is convex and therefore it contains either a unique element or a continuum.
Unfortunately, not all optimal~(or suboptimal ones, corresponding to the same value of $D$) reparametrizations are equally good for our method.
Moreover, different dual algorithms return different reparametrizations and the fastest algorithm may not return an appropriate one.

Therefore, we developed a post-processing algorithm to turn an arbitrary reparametrization into a suitable one without decreasing the value of $D$.
This algorithm consists of two steps: (i)~several iterations of a message passing~(dual block-coordinate ascent) algorithm, which accumulates weights in unary costs and (ii)~partial redistribution of unary costs between incident pairwise cost functions.
This two-step procedure empirically turns most of the nodes, where the LP relaxation~\eqref{equ:LP-formulation} has an integer solution, into strictly arc-consistent ones.
The details of both steps are described in the supplement.

\Paragraph{Higher Order Extensions}
All discussed techniques are easily extended to the higher-order MAP-inference problem
\begin{equation}
  \min_{x\in\SX_\SV} \left[ E_\SG(\theta, x) := \sum_{c\in\SC} \theta_c(x_c) \right]\,.
\end{equation}
where the cliques $c\in\SC\subseteq 2^\SV$ in the decomposition of the energy function $E_\SG(\theta, x)$ may contain terms dependent on 3, 4 and more nodes.
The bound~\eqref{equ:lower-bound} in the higher-order case reads as
\begin{multline}
  \label{equ:lower-bound-higher-order}
  E_\SG(\theta, x'\circ x'') \geq E_\SA(\theta, x') + \\ + E_\SB(\theta, x'') + \sum_{c\in\SC_\SAB} \min_{y_c\in\SX_c} \theta_c(x_c)\,,
\end{multline}
where $\SC_\SAB := \{ c\in\SC \mid \exists u\in\SV_\SA, v\in\SV_\SB\colon u,v\in c \}$ similar to $\SE_\SAB$ in the pairwise case.
Proposition~\ref{prop:lower-bound} for the higher-order case turns into:
\begin{proposition}
  Let $x'\in\argmin_{x\in\SX_\SA} E_\SA(\theta, x)$ and $x''\in\argmin_{x\in\SX_\SB} E_\SB(\theta, x)$.
  The lower bound~\eqref{equ:lower-bound-higher-order} is tight if for all $c\in\SC_\SAB$ and $v\in c$ it holds that $y^*_v = \begin{cases} x'_v\,, & v\in\SV_\SA \\ x''_v\,, & v\in\SV_\SB\end{cases}$ where $y^* \in \argmin_{y\in\SX_\SC} \theta_c(y)$.
\end{proposition}

The proof follows the same reasoning as the proof of Proposition~\ref{prop:lower-bound} and is omitted here.

\begin{table}
  \newcommand\B{\bfseries}
  \newcommand\HM[1]{\multicolumn{1}{c}{#1}}
  \newcommand\ND{\multicolumn{1}{c}{---}}
  \newcommand\I[1]{\fontsize{8pt}{8pt}\selectfont (#1\rlap{)}}
  \centering
  \begin{tabular}{@{}l rrr@{}}
    \toprule
    dataset (avg. $|V|$)                   & \HM{popt} & \HM{clp}  & \HM{dclp}      \\
    \cmidrule(r){1-1}\cmidrule(lr){2-2}\cmidrule(lr){3-3}\cmidrule(l){4-4}
    \Dataset{worms}           \hfill(558)  &     100\% & 69.30\% & \B 26.08\% \\
    \Dataset{protein-folding} \hfill(37)   &   79.22\% &   100\% & \B 71.03\% \\
    \Dataset{color-seg}       \hfill(79k)  &   12.10\% &  0.16\% & \B  0.06\% \\
    \Dataset{mrf-stereo}      \hfill(138k) &   45.19\% & 33.58\% & \B 33.49\% \\[-3pt]
                                           & \I{53.30\%} & \I{0.45\%} & \B \I{0.20\%} \\
    \Dataset{OnCallRostering} \hfill(948)  &       \ND & 98.80\% & \B 65.68\% \\
    \bottomrule
  \end{tabular}
  \caption{
   Average size of the final ILP subproblem~(smaller is better).
    The table shows the percentage of {\em labels} that reside in the ILP subproblem. The measure makes comparable \Algorithm{clp} and \Algorithm{dclp}, which work nodewise, and \Algorithm{popt} working labelwise. See supplement for details. Values in parentheses for \Dataset{mrf-stereo} show the percentage only for the two problem instances, which were solved by \Algorithm{clp} and \Algorithm{dclp}.
  }
  \label{tab:mask-size}
\end{table}

\begin{table*}
  \newcommand\B{\bfseries}
  \newcommand\HS[1]{\multicolumn{1}{c}{#1}}
  \newcommand\HM[1]{\multicolumn{2}{c}{#1}}
  \newcommand\AD{$\frac{2|\SE|}{|V|\cdot(|V|-1)}$}
  \newcommand\ND{\multicolumn{1}{c}{---}}
  \centering
  \begin{tabular}{l r rr rr rr rr rr rr}
    \toprule
    dataset (\#instances)           & \HS{density} & \HM{cpx} & \HM{tb2}       & \HM{popt-tb2}  & \HM{clp-orig} & \HM{clp-tb2}   & \HM{dclp-tb2} \\
    \cmidrule(r){1-2}\cmidrule(lr){3-4}\cmidrule(lr){5-6}\cmidrule(lr){7-8}\cmidrule(lr){9-10}\cmidrule(lr){11-12}\cmidrule(l){13-14}
    \Dataset{worms} (30)            &     10.6\,\% & 1 & 54.7 &    13 &    8.3 &    13 &    8.0 &     15 & 14.2 &    17 &    6.9 & \B 25 & \B  5.8 \\[-3pt]
                                    &&&&&&&&&&&& \fontsize{8pt}{8pt}\selectfont (17\rlap{)} & \fontsize{8pt}{8pt}\selectfont (3.1\rlap{)} \\[-1pt]
    \Dataset{protein-folding} (11)  &      100\,\% & 2 & 48.5 & \B 11 &    1.1 & \B 11 &    1.7 &     10 & 16.8 & \B 11 &    0.9 & \B 11 & \B  0.8 \\
    \cmidrule(lr){1-2}
    \Dataset{color-seg} (19)        &    0.007\,\% & 5 &  4.9 &    15 &   22.1 & \B 18 & \B 0.3 &  \B 18 &  7.6 & \B 18 &    1.1 & \B 18 &     1.4 \\
    \Dataset{mrf-stereo} (3)        &    0.003\,\% & 0 &  \ND &     0 &   \ND  &     1 &    0.9 &  \B  2 & 46.9 & \B  2 &    3.2 & \B  2 & \B  2.3 \\
    \cmidrule(lr){1-2}
    \Dataset{OnCallRostering} (3)   &      0.9\,\% & 2 &  0.9 &     2 &    0.1 &   \ND &    \ND &    \ND &  \ND & \B  3 &    2.3 & \B  3 & \B  1.1 \\
    \bottomrule
  \end{tabular}
  \caption{
    Overview of benchmark results.
    For each method the left number displays the \emph{number of solved instances} and the right one the \emph{average runtime in minutes} for solved instances.
    We computed the graph density as $\frac{2|\SE|}{|\SV|(|\SV|-1)}$ for pairwise models and as $\frac{|\{ (u,v)\in\SV^2 \mid \exists c\in\SC\colon u\in c, v\in c, u\neq v \}|}{|\SV|(|\SV|-1)}$ for \Dataset{OnCallRostering}.
    Values in parentheses for \Algorithm{dclp-tb2} show the average running time on the $17$ problem instances solved by the best competitor \Algorithm{clp-tb2}.
  }
  \label{tab:performance}
\end{table*}

\section{Experimental Evaluation}

\Paragraph{Algorithms}
In this section we compare our proposed algorithm with other related methods.
As baselines we use CPLEX 12.6.2~\cite{cplex} and ToulBar2 0.9.8.0~\cite{cooper2010soft} where the first is the well-known commercial optimizer and the latter is one of the best dedicated branch-and-bound solvers for~\eqref{equ:energy-min}, see comparison in~\cite{hurley2016multi}.
We used comparable parameters and settings like the ones used in~\cite{hurley2016multi}.
They are denoted by \Algorithm{cpx} or \Algorithm{tb2} respectively.
The original CombiLP~\cite{savchynskyy2013global} implementation is referred as \Algorithm{clp-orig}. For a fair comparison, we modified it to make it compatible with arbitrary LP and ILP solvers, in particular, by applying the reparametrization post-processing algorithm described above. The modified method referred as \Algorithm{clp} is up to an order of magnitude {\em faster} than the original one~\Algorithm{clp-orig} (see Table~\ref{tab:performance}).
For the experiments with \Algorithm{clp} and \Algorithm{dclp} we used both CPLEX and ToulBar2 as ILP-solvers. The corresponding variants of \Algorithm{clp} are denoted as \Algorithm{clp-cpx} and \Algorithm{clp-tb2} respectively and similarly for \Algorithm{dclp}.
Since the ToulBar2 variants~(\Algorithm{clp-tb2} and \Algorithm{dclp-tb2}) were superior to the CPLEX variants in \emph{all} our tests, we will mainly discuss the former here~(see supplement for all results).
\TRWS~\cite{kolmorogrov2006convergent} is used as fast block-coordinate-descend LP-solver everywhere except higher-order models.
We used a fast implementation of the solver from the work~\cite{shekhovtsov2015maximum}.
Only for higher-order examples we resort to SRMP~\cite{kolmogorov2015new} using the \emph{minimal} or \emph{basic LP relaxation}~(for details see\ \citealt{kolmogorov2015new}).
We set the maximum number of \TRWS/SRMP iterations to $2000$.
Furthermore we tested the performance of a recent partial optimality technique~\cite{shekhovtsov2015maximum} which is denoted by \Algorithm{popt}.
As this approach does not solve the whole problem, we run ToulBar2 on the reduced model and measure the total running time~(\Algorithm{popt-tb2}).
We set the maximal running time for all methods to $1$ hour.

\begin{figure}
  \centering
  \includegraphics{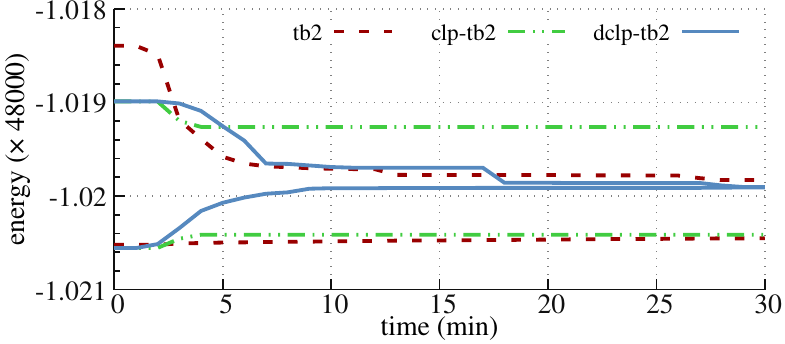}
  \caption{Primal and dual bound progress for \Dataset{worms}.
  All values are averaged over instances where at least one of the solvers returned the optimal solution~($25$~instances).
  To improve clarity only the three best solvers are shown.}
  \label{fig:value-bound}
\end{figure}

\Paragraph{Datasets}
We verify performance of the algorithms on the following publicly available datasets: \Dataset{worms} \cite{dataset-worm}, \Dataset{color-seg} \cite{dataset-colorseg}, \Dataset{mrf-stereo} \cite{dataset-mrfstereo} and \Dataset{OnCallRostering} \cite{dataset-minizinc}, \Dataset{protein\hyp{}folding} \cite{dataset-proteinfolding}.
Each of these datasets is included to highlight specific strengths and weaknesses of the competing methods.
The \Dataset{worms} dataset~(30~instances) serves as a prime example for our algorithm due to its relatively densely connected graph structure and a small duality gap.
The \Dataset{mrf-stereo}~(3~instances) and \Dataset{color-seg}~(19~instances) datasets consist of sparsely connected grid-models and are used to compare performance to the CombiLP method \Algorithm{clp}.
The \Dataset{protein\hyp{}folding} dataset can be split into easy problems~(many nodes, sparsely connected) and hard problems~(only around 33-40 nodes, fully connected).
In the following, we only consider the hard problems~($11$ instances in total).
Last but not least, the dataset \Dataset{OnCallRostering}~($3$ instances) is included as an example of higher-order models, which include cliques of order four.
Unfortunately, we were unable to convert other instances of this dataset from the benchmark~\cite{hurley2016multi} because of a memory bottlenecks in the conversion process.
Apart from \Dataset{OnCallRostering} and \Dataset{worms}, all other problem instances were taken from the OpenGM benchmark~\cite{kappes2015comparative}.

\Paragraph{Results}
We compare and analyse performance of our method in the following three settings:
\emph{(i)}~targeted dense problems like the \Dataset{worms} and \Dataset{protein-folding} datasets;
\emph{(ii)}~sparse problems~(\Dataset{mrf-stereo} and \Dataset{color-seg}), and
\emph{(iii)}~exemplary higher-order problems~(\Dataset{OnCallRostering}).

 \emph{(i)~dense models:} On the \Dataset{worms} dataset our method \Algorithm{dclp-tb2} clearly outperforms competitors, as Table~\ref{tab:performance} shows.
 \Algorithm{dclp-tb2} solves $25$ instances out of $30$, the next competitor \Algorithm{clp-tb2} -- only 17. Moreover, our solver is also more than $2$ times faster than \Algorithm{clp-tb2} in average. This is due to the fact that the resulting ILP subproblem of \Algorithm{dclp} is much smaller that those of \Algorithm{clp}, see Figure~\ref{fig:masks} for visual comparison. The partial optimality method is unable to reduce the problem (see Table~\ref{tab:mask-size}) because of infinite pairwise costs to disallow assigning the same label to different nodes. 
 Figure~\ref{fig:value-bound} shows primal and dual bounds as a function of computational time for this dataset.

 Although on the \Dataset{protein-folding} dataset \Algorithm{dclp-tb2} also outperforms all its competitors, the improvement over \Algorithm{clp-tb2} and \Algorithm{tb2} is not that pronouncing as for the \Dataset{worms} dataset. This is because the final ILP subproblem of \Algorithm{dclp} covers a significant part of the whole graph (over $71\%$ in average). To satisfy its optimality criterion, \Algorithm{dclp} performs up to $5$ iterations with smaller ILP subproblems.  
 In contrast, \Algorithm{clp} considers the whole graph as an ILP subproblem right at the very first iteration. Interestingly that even under this circumstances \Algorithm{clp-tb2} outperforms \Algorithm{tb2} and \Algorithm{clp-cpx} outperforms \Algorithm{cpx} (see supplement for details). The latter solves only $2$ problem instances out of $11$, whereas \Algorithm{clp-cpx} is able to cope with $9$. We attribute it to the reparametrization, which is performed by \Algorithm{clp} prior to passing the problem to \Algorithm{cpx} or \Algorithm{tb2} and plays a role of an efficient presolving. 

\emph{(ii)~sparse models:} Sparse (grid-structured) datasets \Dataset{mrf\hyp{}stereo} and \Dataset{color\hyp{}seg} with about $10^5$ graph nodes each are very well suitable for both \Algorithm{clp} and \Algorithm{dclp} methods and are difficult for \Algorithm{cpx} and \Algorithm{tb2}. Both \Algorithm{clp} and \Algorithm{dclp} are able to solve all the problems except the largest one (\Dataset{teddy} from \Dataset{mrf-stereo} dataset with over $1.6\times 10^5$ nodes and $60$ labels) in similar time. On \Dataset{color-seg} the method \Algorithm{clp-tb2} is somewhat faster, whereas \Algorithm{dclp-tb2} requires less time on \Dataset{mrf-stereo}. This is due to the fact that \Algorithm{dclp} consistently produces smaller ILP subproblems (see Table~\ref{tab:mask-size} for comparison), but \Algorithm{clp} may require less iterations due to the start with a larger ILP subproblem. Partial optimality \Algorithm{popt-tb2} is the winner for the \Dataset{color-seg} dataset: Although its ILP subproblems are larger than those of \Algorithm{clp} and \Algorithm{dclp}, it runs an ILP solver only once. However, results of \Algorithm{popt-tb2} on \Dataset{mrf-stereo} are useful only up to a limited extend: They are sufficient to solve only a single, the simplest problem from that dataset (\Dataset{tsukuba}). \Algorithm{dclp-tb2} and \Algorithm{clp-tb2} in contrast solve two problem instances each.

\emph{(iii)~higher-order models:}
The dataset \Dataset{On\-Call\-Rostering} is included mainly to show applicability of our method to higher-order models.
Generally, higher-order models pose additional difficulties to solvers because they are intrinsically dense and the size of an ILP formulation of the problem grows exponentially with the problem order, therefore even small problems may not fit into memory of an ILP solver.
The \Algorithm{dclp} method again shows its advantage over \Algorithm{clp} as similarly as in the case of the \Dataset{worms} dataset: Since the problems are intrinsically dense, the ILP subproblem for \Algorithm{dclp} is smaller, which results in $2\times$ speed-up compared to \Algorithm{clp}. We also found \Algorithm{tb2} and \Algorithm{cpx} to be quite efficient on this dataset, although they were able to solve only $2$ problems out of $3$.

\section{Conclusions}

We presented a new method, suitable to solve efficiently large-scale MAP-inference problems.
The prerequisites for efficiency is the ``almost'' tight LP relaxation, \ie\ the non-strict-arc-consistent subset of nodes should constitute only a small portion of the problem.
In this case, it isnot the size of the problem which important, but only the size of its non-strict-arc-consistent subproblem.
Comparing to previous works, our method is able to further reduce this size, which is especially notable if the underlying graph structure of the model is non-sparse.
In the future, we plan to extend the method to a broader class of combinatorial problems.

\section*{Acknowledgement}

This work was supported by the DFG grant ``Exact Re\-lax\-a\-tion-Based Inference in Graphical Models'' (SA~2640/1-1).
We thank the Center for Information Services and High Performance Computing (ZIH) at TU Dresden for generous allocations of computer time.

\bibliography{paper}
\bibliographystyle{aaai}

\clearpage
\input{appendix.tex}

\end{document}

%% file: appendix.tex
\twocolumn[{%
  \section{Supplementary Material for\\``Exact MAP-Inference by Confining Combinatorial Search with LP Relaxation''}%
  \vspace{1em}%
  \begin{center}%
    \def\thefootnote{\fnsymbol{footnote}}%
    Stefan Haller\footnotemark[2],
    Paul Swoboda\footnotemark[3],
    Bogdan Savchynskyy\footnotemark[2]\\
    \footnotemark[2] University of Heidelberg,\hspace{1em}%
    \footnotemark[3] IST Austria\\
    \texttt{\small stefan.haller@iwr.uni-heidelberg.de}%
  \end{center}%
  \vspace{2em}%
}]
\raggedbottom

\Paragraph{Proof of Proposition~\ref{prop:lower-bound}}

\begin{proof}
  From $(y^*_u, y^*_v)\in\argmin_{(s,t)\in\SX_{uv}} \theta_{uv}(s, t)$ it follows that $\min_{(s, t)\in\SX_{uv}} \theta_{uv}(s, t) = \theta_{uv}(y^*_u, y^*_v)$ for all $uv\in\SE_\SAB$.
  Hence the right-hand-sides of \eqref{equ:energy-decomposition} and \eqref{equ:lower-bound} are equal.
\end{proof}

\Paragraph{Proof of Proposition~\ref{prop:greedy}}

\begin{lemma}\label{lem:proof-greedy-1}
  From requirements of Proposition~\ref{prop:lower-bound} follows that $x' \circ x'' = y' \circ y''$.
\end{lemma}
\begin{proof}\label{lem:proof-greedy-2}
  Since $\SA'\subseteq\SA$ and $\SV_\SA\subseteq\SAC(\theta)$, it holds that $\SV_{\SA'} \subseteq \SAC(\theta)$. From strict arc-consistency we know that $x'$ and $y'$ are determined by~\eqref{equ:locally-optimal-solution}, hence $x'_v = y'_v$ for all $v\in \SV_{\SA'}$.

  From decomposition~\eqref{equ:energy-decomposition} we known that $E_{\SB'}(\theta, y'') \ge E_{\SB'\cap\SA,B}(\theta, y'') + E_\SB(\theta, y'') + \sum_{uv\in\SE_{\SB'\cap\SA,\SB}} \theta(y''_u, y''_v) \ge E_{\SB'\cap\SA,B}(\theta, x') + E_\SB(\theta, x'') + \sum_{uv\in\SE_{\SB'\cap\SA,\SB}} \min_{(s,t)\in\SX_{uv}}\theta(s,t) = E_{\SB'}(\theta, x' \circ x'')$ where the last equality holds due to Proposition~\ref{prop:lower-bound}.
  From optimality of $y''$ we know that $E_{\SB'}(\theta, x' \circ x'') \ge E_{\SB'}(\theta, y'')$, hence $(x'\circ x'')_v = y''_v$ for all $v\in\SV_{\SB'}$ as there is only one unique solution.
\end{proof}

Proof of Proposition~\ref{prop:greedy}:

\begin{proof}
  From the requirements of the Proposition we already know that $y'$ and $y''$ are the optimal assignment of $\SA'$ and $\SB'$ respectively.
  It remains to show that $\theta_{uv}(y'_u, y''_v) = \min_{(s,t)\in\SX_{uv}} \theta_{uv}(s,t)$ for all $uv\in\SE_{\SA'\SB'}$.

  Applying~\eqref{equ:energy-decomposition} to $y' \circ y''$ results in $E_\SG(\theta, y' \circ y'') = E_{\SA'}(\theta, y') + E_{\SB'}(\theta, y'') + \sum_{uv\in\SE_{\SA'\SB'}} \theta_{uv}(y'_u, y''_v)$.
  Due to $\SA'\subseteq\SA$ if $uv\in\SE_{\SA'\SB'}$ either $u,v\in\SV_\SA$ or $uv\in\SE_\SAB$.
  In the first case it follows from $\SV_\SA\subseteq\SAC(\theta)$, as for the optimal strictly arc-consistent label $x'$ for $\SA'$ it holds that $\theta(x'_u, x'_v) = \min_{(s,t)\in\SX_{uv}} \theta_{uv}(s,t)$ and $(x'_u, x'_v) = (y'_u, y''_v)$ (Lemma~\ref{lem:proof-greedy-1}).
  In the second case $(y'_u, y''_v) = (x'_u, x''_v)$ (Lemma~\ref{lem:proof-greedy-1}) and from fulfillment of Proposition~\eqref{prop:lower-bound} for $x'\circ x''$ follows that $\theta_{uv}(y'_u, y''_v) = \min_{(s,t)\in\SX_{uv}}\theta_{uv}(s,t)$.

  Hence $\theta_{uv}(y'_u, y''_v) = \min_{(s,t)\in\SX_{uv}}\theta_{uv}(s,t)$ for $uv\in\SE_{\SA'\SB'}$ and all requirements of Proposition~\eqref{prop:lower-bound} are fulfilled for $y' \circ y''$.
\end{proof}

\Paragraph{Proof of Proposition~\ref{prop:old-criterion}}

\begin{proof}
  The prerequisites already assure that $x''$ is optimal for $\SB$, so it remains to show that $x'_\SA$ is optimal for $\SA$ and that $(x'_u, x''_v) \in \argmin_{(s,t)\in\SX_{uv}} \theta(s,t)$ for all $uv\in\SE_{\SA\SB}$.
  The optimality of $x'_\SA$ for $\SA$ follows trivially from $\SA\subseteq\SA'\subseteq\SAC(\theta)$ and the fact that $x'$ is optimal for $\SA'$, as for both $\SA$ and $\SA'$ the optimal labeling is determined by~\eqref{equ:locally-optimal-solution}.
  From Definition~\ref{def:boundary-complement} we know that $\SE_{\SA\SB} \subseteq \{uv\in\SE \mid u\in\SV_{\SA'}, v\in\SV_\DA\}$.
  In other words, all edges $\SE_{\SA\SB}$ are covered by subgraph $\SA'$ (note that $\SV_\DA \subseteq \SA'$).
  Due to $\SV_{\SA'} \subseteq \SAC(\theta)$, for all $uv\in\SE_{\SA\SB}$ it is true that $\theta_{uv}(x'_u, x'_v) = \min_{(s,t)\in\SX_{uv}} \theta(s, t)$.
  From the preconditions we know that $x'_v = x''_v$ for all $v\in\SV_\DA$, hence $(x'_u, x''_v) \in \argmin_{(s,t)\in\SX_{uv}} \theta(s, t)$.
\end{proof}

\Paragraph{Reparametrization Post-Processing}

The details of both steps are described below.

(i)~In order to obtain a fast post-processing algorithm we modified one of the fastest dual methods, \TRWS\ of~\cite{kolmorogrov2006convergent}, which also can be seen as a special case of its higher-order counterpart SRMP~\cite{kolmogorov2015new}.
For the sake of brevity we refer to the latter method, because of its simpler presentation and because it works also for higher order models~(see below).
Our whole modification consisted in reassigning the weights $\omega^+_{\alpha, \beta}$ defined by expression~(14) in~\cite{kolmogorov2015new} with the values
\begin{equation}\label{equ:SRMP-omega-modified}
 \omega^+_{\alpha, \beta}=\hspace{-2pt}
\begin{cases}
\frac{1}{|O^+_{\beta}|+\max\{|I^+_{\beta}|,|I_{\beta}-I^+_{\beta}|\} + \lambda} & \text{if}\ (\alpha,\beta)\in I^+_{\beta}\\
0, & \text{if}\ (\alpha,\beta)\in I_{\beta}-I^+_{\beta}\,.
\end{cases}
\end{equation}
We also performed the same reassignment of the weights $\omega^-_{\alpha, \beta}$~(defined by expression~(16) in~\cite{kolmogorov2015new}), with $I^-_{\beta}$ in place of $I^+_{\beta}$.
We refer to~\cite{kolmogorov2015new} for a detailed description of the notation and the algorithm itself.
The only difference between~\eqref{equ:SRMP-omega-modified} and the original expression~(14) from~\cite{kolmogorov2015new} is an additional term $\lambda > 0$ added to the denominator of the expression in the upper line of~\eqref{equ:SRMP-omega-modified}.
The non-zero $\lambda$ leads to redistribution of the labeling costs between unary cost functions.
Therefore, non-optimal labels get higher costs than those belonging to an optimal labeling.
We empirically found the value $\lambda=0.1$ to work well in practice.

(ii)~It is a property of a~(modified) SRMP method that locally optimal pairwise costs $\min_{(s,t)\in\SX_{uv}}\theta_{uv}(s,t)$ are always $0$ for any graph edge $uv\in\SE_\SG$.
The mentioned above partial redistribution of unary costs between incident pairwise cost functions was done as $\theta_{uv}(s,t)\mathop{{+}{=}}\frac{\theta_u(s)}{\nb(u)+1}$ for all $u\in\SV_\SG$ and $s\in\SX_u$.

\Paragraph{Labelwise relative ILP size}
As the partial optimality techniques work on a labelwise basis, we use a labelwise measure for comparing the size of the final ILP subproblem.
For \Algorithm{popt} we use the formula~(35) of~\cite{shekhovtsov2015maximum} to compute the relative number of eliminated labels. Subtracting this value from 100\% yields the values in Table~\ref{tab:mask-size}. The final formula for \Algorithm{popt} looks like the following
\begin{equation}
  1 - \frac{\sum_{v\in\SV}|\SX_v \setminus p_v(\SX_v)|}{\sum_{u\in\SV}(|\SX_v| -1)}\,.
\end{equation}

For \Algorithm{clp} and \Algorithm{dclp} we evaluate the following expression to compute the value with the same semantic:
\begin{equation}
  1 - \frac{\sum_{v\in\SV_\SA}(|\SX_v| - 1)}{\sum_{u\in\SV_\SG}(|\SX_v| -1)}\,.
\end{equation}

As \Algorithm{popt} is a polynomial time algorithm, it will output a reduced model for all instance of the benchmark. As \Algorithm{clp} and \Algorithm{dclp} try to solve the NP-hard MAP-inference problem~\eqref{equ:energy-min}, they do not terminate for all instances. As the maximal ILP subproblem is only defined after Proposition~\ref{prop:lower-bound} holds, we assume the worst-case and use 100\% as ILP subproblem size for unsolved instances.

\onecolumn
\Paragraph{Complete benchmark table}
For lack of space we removed some solvers from the experimental evaluation. The following tables show the results for each instance and each solver separately..

\begin{table}[H]
  \fontsize{8pt}{8pt}\selectfont\centering
  \newcommand\HS[1]{\multicolumn{1}{c}{#1}}
  \newcommand\HM[1]{\multicolumn{2}{c}{#1}}
  \newcommand\ND{\multicolumn{1}{c}{---}}
  \begin{tabular}{L{90pt} R{14pt} R{14pt} R{21pt}R{14pt} R{21pt}R{14pt} R{21pt}R{14pt} R{21pt}R{14pt} R{21pt}R{14pt} R{21pt}R{14pt}}
    \toprule
    instance               & \HS{cpx} & \HS{tb2} & \HM{popt-tb2} & \HM{clp-orig} & \HM{clp-cpx} & \HM{clp-tb2}  & \HM{dclp-cpx} & \HM{dclp-tb2} \\
    \cmidrule(r){1-1}\cmidrule(lr){2-2}\cmidrule(lr){3-3}\cmidrule(lr){4-5}\cmidrule(lr){6-7}\cmidrule(lr){8-9}\cmidrule(lr){10-11}\cmidrule(lr){12-13}\cmidrule(l){14-15}
    C18G1\_2L1\_1          &  \ND     &  \ND     & 100\%  & \ND  & \ND    & \ND  & \ND    & \ND  & \ND    & \ND  & 19.8\% & 58.1 & 19.8\% & 14.8 \\
    cnd1threeL1\_1213061   &  \ND     &  1.3     & 100\%  & 1.1  & 34.1\% & 8.9  & 34.6\% & 1.7  & 34.6\% & 0.6  & 3.9\%  & 0.5  & 3.9\%  & 0.4  \\
    cnd1threeL1\_1228061   &  \ND     &  0.6     & 100\%  & 1.2  & 36.7\% & 7.2  & 35.4\% & 1.8  & 35.4\% & 0.5  & 5.3\%  & 0.4  & 5.3\%  & 0.5  \\
    cnd1threeL1\_1229061   &  \ND     &  \ND     & 100\%  & \ND  & \ND    & \ND  & \ND    & \ND  & \ND    & \ND  & 20.8\% & 56.4 & 20.8\% & 20.5 \\
    cnd1threeL1\_1229062   &  \ND     &  \ND     & 100\%  & \ND  & \ND    & \ND  & \ND    & \ND  & \ND    & \ND  & \ND    & \ND  & \ND    & \ND  \\
    cnd1threeL1\_1229063   &  \ND     &  1.7     & 100\%  & 1.3  & 31.7\% & 15.3 & 26.3\% & 6.3  & 26.3\% & 0.5  & 5.3\%  & 0.6  & 5.3\%  & 0.4  \\
    eft3RW10035L1\_0125071 &  \ND     &  \ND     & 100\%  & \ND  & \ND    & \ND  & \ND    & \ND  & 52.8\% & 51.2 & 16.7\% & 50.1 & 16.7\% & 32.4 \\
    eft3RW10035L1\_0125072 &  \ND     &  \ND     & 100\%  & \ND  & \ND    & \ND  & \ND    & \ND  & \ND    & \ND  & 16.8\% & 29.0 & 16.8\% & 1.9  \\
    eft3RW10035L1\_0125073 &  \ND     &  \ND     & 100\%  & \ND  & \ND    & \ND  & \ND    & \ND  & \ND    & \ND  & \ND    & \ND  & \ND    & \ND  \\
    egl5L1\_0606074        &  \ND     &  \ND     & 100\%  & \ND  & \ND    & \ND  & \ND    & \ND  & \ND    & \ND  & \ND    & \ND  & \ND    & \ND  \\
    elt3L1\_0503071        &  \ND     & 51.2     & 100\%  & 55.5 & 61.5\% & 36.5 & \ND    & \ND  & 60.8\% & 2.3  & 14.6\% & 3.3  & 14.6\% & 0.7  \\
    elt3L1\_0503072        &  \ND     &  9.3     & 100\%  & 4.3  & 53.2\% & 15.0 & 49.5\% & 10.6 & 49.5\% & 1.2  & 10.0\% & 0.7  & 10.0\% & 0.4  \\
    elt3L1\_0504073        &  \ND     &  \ND     & 100\%  & \ND  & \ND    & \ND  & \ND    & \ND  & \ND    & \ND  & 13.4\% & 1.7  & 13.4\% & 0.9  \\
    hlh1fourL1\_0417071    &  \ND     &  \ND     & 100\%  & \ND  & \ND    & \ND  & \ND    & \ND  & \ND    & \ND  & \ND    & \ND  & \ND    & \ND  \\
    hlh1fourL1\_0417075    &  \ND     &  1.8     & 100\%  & 1.3  & 44.7\% & 11.5 & 45.8\% & 6.3  & 45.8\% & 0.6  & 5.2\%  & 0.5  & 5.2\%  & 0.5  \\
    hlh1fourL1\_0417076    &  \ND     &  \ND     & 100\%  & \ND  & \ND    & \ND  & \ND    & \ND  & 68.4\% & 52.8 & 20.2\% & 25.9 & 20.2\% & 13.0 \\
    hlh1fourL1\_0417077    &  \ND     & 28.8     & 100\%  & 30.8 & 46.7\% & 7.1  & 46.7\% & 3.9  & 46.7\% & 0.7  & 5.6\%  & 0.5  & 5.6\%  & 0.5  \\
    hlh1fourL1\_0417078    &  \ND     &  3.5     & 100\%  & 2.0  & 54.1\% & 17.8 & 53.0\% & 30.1 & 53.0\% & 1.0  & 10.4\% & 1.8  & 10.4\% & 0.5  \\
    mir61L1\_1228061       &  \ND     &  \ND     & 100\%  & \ND  & \ND    & \ND  & \ND    & \ND  & \ND    & \ND  & 13.5\% & 22.7 & 13.5\% & 5.1  \\
    mir61L1\_1228062       &  \ND     &  \ND     & 100\%  & \ND  & \ND    & \ND  & \ND    & \ND  & \ND    & \ND  & \ND    & \ND  & 15.9\% & 39.2 \\
    mir61L1\_1229062       &  \ND     &  \ND     & 100\%  & \ND  & 67.1\% & 18.0 & 67.5\% & 14.3 & 67.5\% & 1.1  & 8.9\%  & 1.3  & 8.9\%  & 0.6  \\
    pha4A7L1\_1213061      &  \ND     &  \ND     & 100\%  & \ND  & \ND    & \ND  & \ND    & \ND  & \ND    & \ND  & 20.1\% & 17.0 & 20.1\% & 7.1  \\
    pha4A7L1\_1213062      &  54.7    &  1.2     & 100\%  & 1.0  & 8.4\%  & 8.2  & 8.4\%  & 0.4  & 8.4\%  & 0.3  & 0.7\%  & 0.3  & 0.7\%  & 0.3  \\
    pha4A7L1\_1213064      &  \ND     &  \ND     & 100\%  & \ND  & 47.4\% & 20.9 & 40.8\% & 16.8 & 40.8\% & 1.4  & 11.1\% & 1.2  & 11.1\% & 0.6  \\
    pha4B2L1\_0125072      &  \ND     &  \ND     & 100\%  & \ND  & \ND    & \ND  & \ND    & \ND  & \ND    & \ND  & 20.1\% & 10.7 & 20.1\% & 2.3  \\
    pha4I2L\_0408071       &  \ND     &  \ND     & 100\%  & \ND  & \ND    & \ND  & \ND    & \ND  & \ND    & \ND  & \ND    & \ND  & \ND    & \ND  \\
    pha4I2L\_0408072       &  \ND     &  1.9     & 100\%  & 1.3  & 42.5\% & 9.1  & 43.2\% & 4.2  & 43.2\% & 0.6  & 4.9\%  & 0.5  & 4.9\%  & 0.4  \\
    pha4I2L\_0408073       &  \ND     &  3.0     & 100\%  & 1.8  & 60.2\% & 11.9 & 60.5\% & 7.6  & 60.5\% & 1.0  & 10.7\% & 1.2  & 10.7\% & 0.5  \\
    unc54L1\_0123071       &  \ND     &  1.5     & 100\%  & 1.2  & 32.3\% & 10.8 & 32.3\% & 1.5  & 32.3\% & 0.6  & 1.9\%  & 0.5  & 1.9\%  & 0.4  \\
    unc54L1\_0123072       &  \ND     &  1.8     & 100\%  & 1.3  & 53.9\% & 12.0 & 53.9\% & 4.2  & 53.9\% & 0.7  & 7.9\%  & 0.6  & 7.9\%  & 0.6  \\
    \midrule
    average                &  54.7    &  8.3     & 100\%  & 8.0  & 50.0\% & 14.0 & 42.7\% & 7.8  & 45.8\% & 6.9  & 11.2\% & 11.9 & 11.3\% & 5.8 \\
    \bottomrule
  \end{tabular}
  \caption{Complete benchmark results for \Dataset{worms} dataset.}
\end{table}

\begin{table}[H]
  \fontsize{8pt}{8pt}\selectfont\centering
  \newcommand\HS[1]{\multicolumn{1}{c}{#1}}
  \newcommand\HM[1]{\multicolumn{2}{c}{#1}}
  \newcommand\ND{\multicolumn{1}{c}{---}}
  \begin{tabular}{L{90pt} R{14pt} R{14pt} R{21pt}R{14pt} R{21pt}R{14pt} R{21pt}R{14pt} R{21pt}R{14pt} R{21pt}R{14pt} R{21pt}R{14pt}}
    \toprule
    instance & \HS{cpx} & \HS{tb2} & \HM{popt-tb2} & \HM{clp-orig} & \HM{clp-cpx}  & \HM{clp-tb2} & \HM{dclp-cpx} & \HM{dclp-tb2} \\
    \cmidrule(r){1-1}\cmidrule(lr){2-2}\cmidrule(lr){3-3}\cmidrule(lr){4-5}\cmidrule(lr){6-7}\cmidrule(lr){8-9}\cmidrule(lr){10-11}\cmidrule(lr){12-13}\cmidrule(l){14-15}
    1CKK     & \ND      & 0.5      & 73.6\% & 1.2  & 100\%  & 13.6 &  100\% & 36.0 & 100\% & 0.7  & 76.4\% & 13.2   & 76.4\% & 0.6 \\
    1CM1     & \ND      & 0.6      & 70.1\% & 1.2  & 100\%  & 7.9  &  100\% & 3.1  & 100\% & 0.5  & 42.0\% & 1.4    & 42.0\% & 0.3 \\
    1SY9     & 38.6     & 0.5      & 42.2\% & 0.8  & 100\%  & 8.5  &  100\% & 6.4  & 100\% & 0.6  & 31.0\% & 0.7    & 31.0\% & 0.3 \\
    2BBN     & \ND      & 1.2      & 85.9\% & 2.7  & 100\%  & 31.1 &  100\% & 40.1 & 100\% & 1.2  & 62.9\% & 9.9    & 62.9\% & 1.0 \\
    2BCX     & \ND      & 3.1      & 85.8\% & 3.6  & 100\%  & 31.3 &  \ND   & \ND  & 100\% & 1.6  & \ND    & \ND    & 94.1\% & 2.1 \\
    2BE6     & 58.4     & 0.5      & 84.9\% & 1.1  & 100\%  & 9.3  &  100\% & 13.5 & 100\% & 0.4  & 96.6\% & 7.9    & 96.6\% & 0.7 \\
    2F3Y     & \ND      & 2.7      & 86.3\% & 2.1  & 100\%  & \ND  &  \ND   & \ND  & 100\% & 1.2  & 97.7\% & 50.8   & 97.7\% & 1.5 \\
    2FOT     & \ND      & 1.0      & 89.0\% & 1.5  & \ND    & 18.5 &  100\% & 13.1 & 100\% & 0.8  & 69.7\% & 10.6   & 69.7\% & 0.7 \\
    2HQW     & \ND      & 0.6      & 82.0\% & 1.2  & 100\%  & 10.5 &  100\% & 8.4  & 100\% & 0.8  & 68.2\% & 5.1    & 68.2\% & 0.5 \\
    2O60     & \ND      & 0.8      & 84.5\% & 1.9  & 100\%  & 24.3 &  100\% & 7.9  & 100\% & 0.8  & 91.0\% & 8.1    & 91.0\% & 0.9 \\
    3BXL     & \ND      & 0.7      & 87.7\% & 1.6  & 100\%  & 12.6 &  100\% & 5.2  & 100\% & 0.7  & 52.3\% & 2.4    & 52.3\% & 0.5 \\
    \midrule
    average  & 48.4     & 1.1      & 79.2\% & 1.7  & 100\%  & 21.0 &  100\% & 14.9 & 100\% & 0.9  & 68.8\% & 11.0   & 71.1\% & 0.8 \\
    \bottomrule
  \end{tabular}
  \caption{Complete benchmark results for \Dataset{protein-folding} dataset.}
\end{table}

\begin{table}[H]
  \fontsize{8pt}{8pt}\selectfont\centering
  \newcommand\HS[1]{\multicolumn{1}{c}{#1}}
  \newcommand\HM[1]{\multicolumn{2}{c}{#1}}
  \newcommand\ND{\multicolumn{1}{c}{---}}
  \begin{tabular}{L{90pt} R{14pt} R{14pt} R{21pt}R{14pt} R{21pt}R{14pt} R{21pt}R{14pt} R{21pt}R{14pt} R{21pt}R{14pt} R{21pt}R{14pt}}
    \toprule
    instance & \HS{cpx} & \HS{tb2} & \HM{popt-tb2} & \HM{clp-orig} & \HM{clp-cpx} & \HM{clp-tb2} & \HM{dclp-cpx} & \HM{dclp-tb2} \\
    \cmidrule(r){1-1}\cmidrule(lr){2-2}\cmidrule(lr){3-3}\cmidrule(lr){4-5}\cmidrule(lr){6-7}\cmidrule(lr){8-9}\cmidrule(lr){10-11}\cmidrule(lr){12-13}\cmidrule(l){14-15}
    ted-gm   & \ND      & \ND      & 29.0\%  & \ND & \ND    & \ND  & \ND   & \ND  & \ND   & \ND  & \ND   & \ND   & \ND     & \ND \\
    tsu-gm   & \ND      & \ND      &  6.7\%  & 0.9 & 0.2\%  & 35.0 & 0.2\% & 2.6  & 0.2\% & 1.5  & 0.1\% & 0.8   & 0.1\%   & 0.7 \\
    ven-gm   & \ND      & \ND      & 99.9\%  & \ND & \ND    & 58.7 & 0.7\% & 4.8  & 0.7\% & 4.8  & 0.3\% & 3.9   & 0.3\%   & 3.8 \\
    \midrule
    average  & \ND      & \ND      & 45.2\%  & 0.9 & 0.2\%  & 46.9 & 0.5\% & 3.7  & 0.5\% & 3.2  & 0.2\% & 2.4   & 0.2\%   & 2.3 \\
    \bottomrule
  \end{tabular}
  \caption{Complete benchmark results for \Dataset{mrf-stereo} dataset.}
\end{table}

\begin{table}[H]
  \fontsize{8pt}{8pt}\selectfont\centering
  \newcommand\HS[1]{\multicolumn{1}{c}{#1}}
  \newcommand\HM[1]{\multicolumn{2}{c}{#1}}
  \newcommand\ND{\multicolumn{1}{c}{---}}
  \begin{tabular}{L{90pt} R{14pt} R{14pt} R{21pt}R{14pt} R{21pt}R{14pt} R{21pt}R{14pt} R{21pt}R{14pt} R{21pt}R{14pt} R{21pt}R{14pt}}
    \toprule
    instance                    & \HS{cpx} & \HS{tb2} & \HM{popt-tb2} & \HM{clp-orig} & \HM{clp-cpx} & \HM{clp-tb2} & \HM{dclp-cpx} & \HM{dclp-tb2} \\
    \cmidrule(r){1-1}\cmidrule(lr){2-2}\cmidrule(lr){3-3}\cmidrule(lr){4-5}\cmidrule(lr){6-7}\cmidrule(lr){8-9}\cmidrule(lr){10-11}\cmidrule(lr){12-13}\cmidrule(l){14-15}
    n4/clownfish-small          & \ND      & 23.7     & 8.9\%  & 0.1  & 0.1\%  & 5.4  & 0.1\% & 0.5  & 0.1\% & 1.0  & 0.1\% & 1.0   & 0.1\% & 1.0 \\
    n4/crops-small              & \ND      & 31.1     & 6.8\%  & 0.1  & 0\%    & 0.7  & 0.1\% & 1.0  & 0.1\% & 1.0  & 0.1\% & 1.0   & 0.1\% & 1.0 \\
    n4/fourcolors               & 2.3      & 9.5      & 30.4\% & 0.0  & 0.2\%  & 4.8  & 0.4\% & 0.7  & 0.4\% & 0.7  & 0.1\% & 0.7   & 0.1\% & 0.7 \\
    n4/lake-small               & \ND      & 14.4     & 8.4\%  & 0.1  & 0\%    & 0.2  & 0.0\% & 1.0  & 0.0\% & 1.0  & 0.0\% & 1.0   & 0.0\% & 1.0 \\
    n4/palm-small               & \ND      & 39.3     & 5.2\%  & 0.1  & 0\%    & 0.8  & 0.1\% & 1.0  & 0.1\% & 1.0  & 0.1\% & 1.1   & 0.1\% & 1.0 \\
    n4/penguin-small            & 9.2      & 8.3      & 6.3\%  & 0.0  & 0\%    & 0.3  & 0.0\% & 0.8  & 0.0\% & 0.8  & 0.0\% & 0.8   & 0.0\% & 0.4 \\
    n4/pfau-small               & \ND      & \ND      & 12.7\% & 2.1  & 0.7\%  & 10.6 & 0.5\% & 1.6  & 0.5\% & 0.8  & 0.4\% & 1.6   & 0.4\% & 1.3 \\
    n4/snail                    & 2.1      & 0.5      & 26.7\% & 0.0  & 0.1\%  & 2.9  & 0.1\% & 0.9  & 0.1\% & 0.8  & 0.1\% & 0.8   & 0.1\% & 0.8 \\
    n4/strawberry-glass-2-small & \ND      & 38.9     & 5.8\%  & 0.1  & 0\%    & 4.1  & 0.0\% & 0.9  & 0.0\% & 0.8  & 0.0\% & 0.9   & 0.0\% & 0.9 \\
    n8/clownfish-small          & \ND      & 12.4     & 8.9\%  & 0.2  & 0.2\%  & 11.3 & 0.2\% & 2.3  & 0.2\% & 2.3  & 0.1\% & 2.3   & 0.1\% & 2.1 \\
    n8/crops-small              & \ND      & 54.0     & 6.8\%  & 0.7  & 0.2\%  & 11.6 & 0.4\% & 2.6  & 0.4\% & 2.6  & 0.2\% & 2.5   & 0.2\% & 2.5 \\
    n8/fourcolors               & 7.1      & 15.0     & 30.4\% & 0.1  & 0.5\%  & 5.9  & 0.6\% & 1.5  & 0.6\% & 1.5  & 0.1\% & 1.6   & 0.1\% & 1.4 \\
    n8/lake-small               & \ND      & 14.0     & 8.4\%  & 0.1  & 0.1\%  & 11.8 & 0.1\% & 2.0  & 0.1\% & 2.1  & 0.1\% & 2.0   & 0.1\% & 2.1 \\
    n8/palm-small               & \ND      & \ND      & 5.3\%  & 0.9  & 0.3\%  & 27.1 & 0.3\% & 3.2  & 0.3\% & 3.1  & 0.2\% & 3.5   & 0.2\% & 2.0 \\
    n8/penguin-small            & \ND      & 14.1     & 6.3\%  & 0.2  & 0\%    & 3.4  & 0.0\% & 1.5  & 0.0\% & 1.8  & 0.0\% & 1.8   & 0.0\% & 1.8 \\
    n8/pfau-small               & \ND      & \ND      & 8.8\%  & 1.1  & 0.3\%  & 12.1 & 0.3\% & 2.2  & 0.3\% & 2.1  & 0.2\% & 2.1   & 0.2\% & 2.2 \\
    n8/snail                    & 3.9      & 0.7      & 26.8\% & 0.0  & 0.1\%  & 12.1 & 0.1\% & 1.6  & 0.1\% & 1.6  & 0.1\% & 1.6   & 0.1\% & 1.6 \\
    n8/strawberry-glass-2-small & \ND      & 55.6     & 5.9\%  & 0.6  & 0.4\%  & 12.5 & 0.4\% & 2.3  & 0.4\% & 2.2  & 0.2\% & 2.3   & 0.2\% & 2.2 \\
    \midrule
    average                     & 4.9      & 22.1     & 12.1\% & 0.3  & 0.2\%  & 7.6  & 0.2\% & 1.5  & 0.2\% & 1.1  & 0.1\% & 1.6   & 0.1\% & 1.4 \\
    \bottomrule
  \end{tabular}
  \caption{Complete benchmark results for \Dataset{color-seg} dataset.}
\end{table}

\begin{table}[H]
  \fontsize{8pt}{8pt}\selectfont\centering
  \newcommand\HS[1]{\multicolumn{1}{c}{#1}}
  \newcommand\HM[1]{\multicolumn{2}{c}{#1}}
  \newcommand\ND{\multicolumn{1}{c}{---}}
  \newcommand\NDFa{\multicolumn{1}{>{\Centering}p{14pt}}{---}}
  \newcommand\NDFb{\multicolumn{1}{>{\Centering}p{21pt}}{---}}
  \begin{tabular}{L{90pt} R{14pt} R{14pt} R{21pt}R{14pt} R{21pt}R{14pt} R{21pt}R{14pt} R{21pt}R{14pt} R{21pt}R{14pt} R{21pt}R{14pt}}
    \toprule
    instance & \HS{cpx} & \HS{tb2} & \HM{popt-tb2}  & \HM{clp-orig}  & \HM{clp-cpx} & \HM{clp-tb2} & \HM{dclp-cpx} & \HM{dclp-tb2} \\
    \cmidrule(r){1-1}\cmidrule(lr){2-2}\cmidrule(lr){3-3}\cmidrule(lr){4-5}\cmidrule(lr){6-7}\cmidrule(lr){8-9}\cmidrule(lr){10-11}\cmidrule(lr){12-13}\cmidrule(l){14-15}
    10s-50d  & \ND      & \ND      & \NDFb  & \NDFa & \NDFb  & \NDFa &  \ND   & \ND & 99.9\% & 7.0 & \ND    & \ND  & 42.8\% & 3.4 \\
    4s-10d   & 0.0      & 0.0      & \ND    & \ND   & \ND    & \ND   & 97.7\% & 0.1 & 97.7\% & 0.0 & 71.9\% & 0.1  & 71.9\% & 0.0 \\
    4s-23d   & 1.7      & 0.0      & \ND    & \ND   & \ND    & \ND   & 99.0\% & 8.1 & 99.0\% & 0.0 & 82.5\% & 24.1 & 82.5\% & 0.1 \\
    \midrule
    average  & 0.9      & 0.1      & \ND    & \ND   & \ND    & \ND   & 98.9\% & 4.1 & 98.9\% & 2.3 & 77.2\% & 12.4 & 65.7\% & 1.1 \\
    \bottomrule
  \end{tabular}
  \caption{Complete benchmark results for \Dataset{OnCallRostering} dataset.}
\end{table}